\documentclass[letterpaper]{article} % DO NOT CHANGE THIS
\usepackage{aaai2026}  % DO NOT CHANGE THIS
\usepackage{times}  % DO NOT CHANGE THIS
\usepackage{helvet}  % DO NOT CHANGE THIS
\usepackage{courier}  % DO NOT CHANGE THIS
\usepackage[hyphens]{url}  % DO NOT CHANGE THIS
\usepackage{graphicx} % DO NOT CHANGE THIS
\usepackage{natbib}  % DO NOT CHANGE THIS AND DO NOT ADD ANY OPTIONS TO IT
\usepackage{caption} % DO NOT CHANGE THIS AND DO NOT ADD ANY OPTIONS TO IT
\usepackage{algorithm}
\usepackage{booktabs}
\usepackage{algpseudocode}
\usepackage{tcolorbox}
\usepackage{newfloat}
\usepackage{listings}
\DeclareCaptionStyle{ruled}{labelfont=normalfont,labelsep=colon,strut=off} % DO NOT CHANGE THIS
\floatstyle{ruled}
\newfloat{listing}{tb}{lst}{}
\floatname{listing}{Listing}
\usepackage{multirow}
\usepackage{graphicx}
\usepackage{amsmath}
\usepackage{amssymb}
\usepackage{amsthm}
\usepackage{appendix}
\usepackage{appendix}

\usepackage{subcaption}
\usepackage{xcolor}

\title{CluCERT: Certifying LLM Robustness via Clustering-Guided Denoising Smoothing}
\author{
    Zixia Wang, Gaojie Jin, Jia Hu, Ronghui Mu\thanks{Corresponding Author}
}
\affiliations{
    University of Exeter\\
\{zw483, g.jin, j.hu, r.mu2\}@exeter.ac.uk
}

\begin{document}

\maketitle

\begin{abstract}
Recent advancements in Large Language Models (LLMs) have led to their widespread adoption in daily applications. Despite their impressive capabilities, they remain vulnerable to adversarial attacks, as even minor meaning-preserving changes such as synonym substitutions can lead to incorrect predictions. As a result, certifying the robustness of LLMs against such adversarial prompts is of vital importance. Existing approaches focused on word deletion or simple denoising strategies to achieve robustness certification. However, these methods face two critical limitations: (1) they yield loose robustness bounds due to the lack of semantic validation for perturbed outputs and (2) they suffer from high computational costs due to repeated sampling. To address these limitations, we propose CluCERT, a novel framework for certifying LLM robustness via clustering-guided denoising smoothing. Specifically, to achieve tighter certified bounds, we introduce a semantic clustering filter that reduces noisy samples and retains meaningful perturbations, supported by theoretical analysis. Furthermore, we enhance computational efficiency through two mechanisms: a refine module that extracts core semantics, and a fast synonym substitution strategy that accelerates the denoising process. Finally, we conduct extensive experiments on various downstream tasks and jailbreak defense scenarios. Experimental results demonstrate that our method outperforms existing certified approaches in both robustness bounds and computational efficiency.

\end{abstract}

\section{Introduction}

\begin{figure}[ht]
    \centering
    \includegraphics[width=0.4\textwidth ,trim=180 50 200 30, clip]{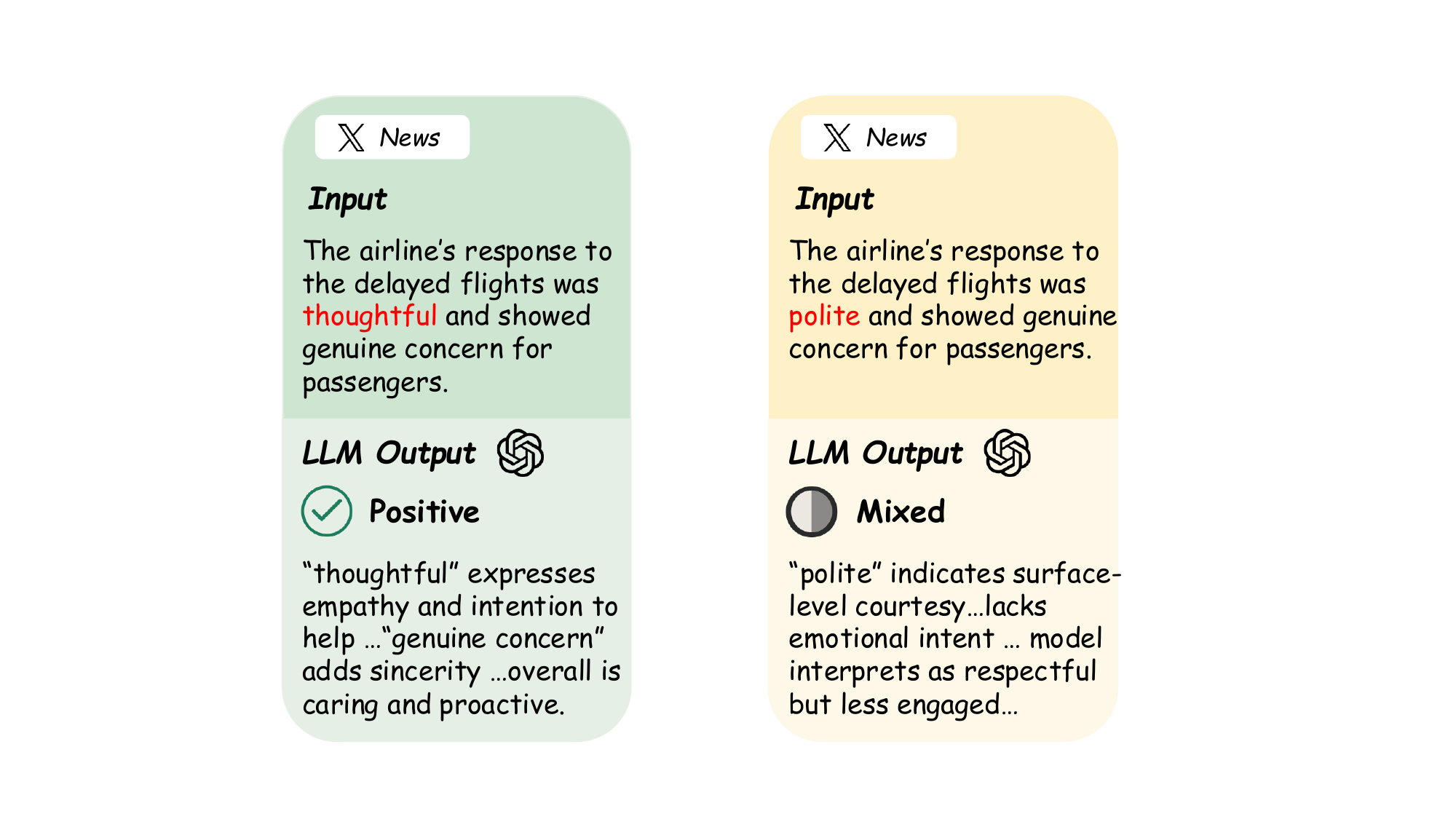}
    \caption{An example showing how a minimal synonym substitution can flip an LLM's sentiment prediction.
}
    \label{fig:start}
\end{figure}
Large Language Models (LLMs) have seen rapid development in recent years. Their strong performance has enabled a wide range of real-world applications~\cite{zhang2025pushing,lin2025explore,chen2024chatunitest}, bringing significant benefits in terms of efficiency and automation.
However, LLMs are vulnerable to adversarial inputs~\cite{shayegani2023survey,yi2024position,huang2024survey,sun-etal-2024-crowd,wang2025safety}, which can mislead predictions \cite{kumar2023certifying, mu2024enhancing}. In particular, in textual tasks, even minor semantic-preserving perturbations such as synonym substitutions~\cite{wang-etal-2021-certified, li2024drattack} or paraphrasing~\cite{fang2025languagemodelsecretlywrite} can easily mislead LLMs into producing incorrect or harmful outputs. This vulnerability becomes particularly pronounced when dealing with texts that contain subtle semantic distinctions~\cite{tao2024robustnesslargelanguagemodels}. For example (as shown in Figure~\ref{fig:start}), the sentence \textit{``The airline's response to the delayed flights was thoughtful...''} is likely to be interpreted as positive by an LLM. However, replacing \textit{``thoughtful''} with a near-synonym such as \textit{``polite''} can lead the model to a semantically distinct and potentially incorrect judgment.

To counter this threat, researchers have proposed various empirical defenses, such as adversarial data augmentation~\cite{cheng2020advaug,11027475} and adversarial fine-tuning~\cite{chen2020adversarial,zhang2024towards}, which have shown some effectiveness in mitigating specific attacks. However, these approaches often lead to a reactive ``arms race''~\cite{jin2020bert}, where evolving attack strategies continually require updated defense mechanisms. As a result, models must be frequently retrained or adapted, leading to an ongoing cycle of patching and evasion. To break this cycle, certified robustness~\cite{raghunathan2018certified,weng2018towards,jin2025reconcile,sun-ruan-2023-textverifier,zhang2025scalable,zhang2024prass,rocamora2025certified} has emerged as a promising alternative. It offers mathematically provable guarantees that model predictions remain consistent within a defined perturbation range. 

Nevertheless, the large number of parameters, high computational cost, and limited access to internal components make many traditional certification methods hard to apply to LLMs. In contrast, randomized smoothing~\cite{cohen2019certifiedadversarialrobustnessrandomized} provides a model-independent solution that uses probabilistic techniques to give robustness guarantees with high confidence. It does not require access to model parameters, making it suitable for black-box LLMs. This concept has been successfully established in the vision domain through randomized smoothing~\cite{cohen2019certifiedadversarialrobustnessrandomized, levine2020robustness, jia2022tightl0normcertifiedrobustness}. Inspired by this, some studies~\cite{jia2019certified,ye2020safer,chao2025jailbreaking} have applied randomized smoothing to natural language processing. 

%While prior efforts have extended certified robustness from vision to text,
While randomized smoothing is widely explored in the vision domain, its application in text is limited. %Existing methods still face key limitations. 
In the previous work, ~\citet{zeng2021certifiedrobustnesstextadversarial} adopt randomized ablation by replacing words with \texttt{[MASK]} tokens. This strategy disrupts semantic and grammatical structure, resulting in certificates that are ineffective against semantics-preserving attacks. Building on this, ~\citet{ji2024advancingrobustnesslargelanguage} use LLMs to fill in masked tokens in an attempt to preserve meaning. However, their approach relies on unverifiable LLM outputs, which may introduce contextually inappropriate substitutions. Moreover, both methods suffer from low efficiency, as large-scale certification and repeated use of LLMs are computationally expensive.

To address these challenges, we propose CluCERT, an efficient certified robustness framework for LLMs, built upon a clustering-guided denoising smoothing strategy (as shown in Figure~\ref{fig:example}). CluCERT aims to enhance certified robustness bounds while maintaining computational efficiency. %The framework consists of four main stages:(a) Refine: The generative capabilities of LLMs are leveraged to perform semantic refinement, where irrelevant tokens are removed from the input to preserve core semantics.
To this end, we highlight two key components that contribute to \textbf{efficiency}. First, we introduce the semantic refinement, where irrelevant tokens are removed based on their contribution to the input’s core meaning. This results in a fixed-length, refined input that avoids repeated sampling on semantically unimportant words. Importantly, the certified radius is computed over this filtered input, ensuring that the robustness guarantee aligns with the core semantic content. Second, instead of relying on LLM-generated token replacements, we introduce a lightweight synonym substitution strategy based on WordNet and embedding similarity, which avoids model queries while preserving semantic consistency.  To further \textbf{denoise}, we apply semantic clustering over the perturbed samples,  filtering out samples that are semantically inconsistent with the original input. This process increases the probability of the most frequent response while reducing the impact of less frequent ones.
%(b) Denoise: Multiple perturbed samples are generated via random synonym substitutions, followed by semantic clustering—based on sentence similarity—to filter out noisy samples while retaining meaningful perturbations.
%(c) Predict and (d) Certify: Different LLMs are employed for classification, and the final certified radius is computed using majority voting over the aggregated predictions. 

This paper's main \textbf{contributions} are summarized as follows:
%\begin{itemize}
    %\item   
    (i) We develop a fast synonym substitution strategy combined with a semantic refinement module that removes less informative tokens. This design significantly improves computational efficiency and reduces the overall cost of certification.
    %\item  
(ii) We propose a novel framework, CluCERT, for certifying the robustness of LLMs. It incorporates a clustering-guided denoising module that preserves only meaning-consistent perturbations, resulting in tighter certified robustness bounds. We provide both theoretical analysis and empirical results to validate the effectiveness of our approach.
    %\item 
    (iii) We conduct extensive experiments across multiple tasks to demonstrate the effectiveness and efficiency of our approach. To the best of our knowledge, we are the first to apply certified robustness techniques to math word problem solving, a domain that requires precise semantics and sensitive to input perturbations. %math word problem solving, and jailbreak defense, to demonstrate the effectiveness of our framework.
%\end{itemize}

\begin{figure*}[ht] 
    \centering 
\includegraphics[width=0.9\textwidth,trim=20 30 10 45, clip]
{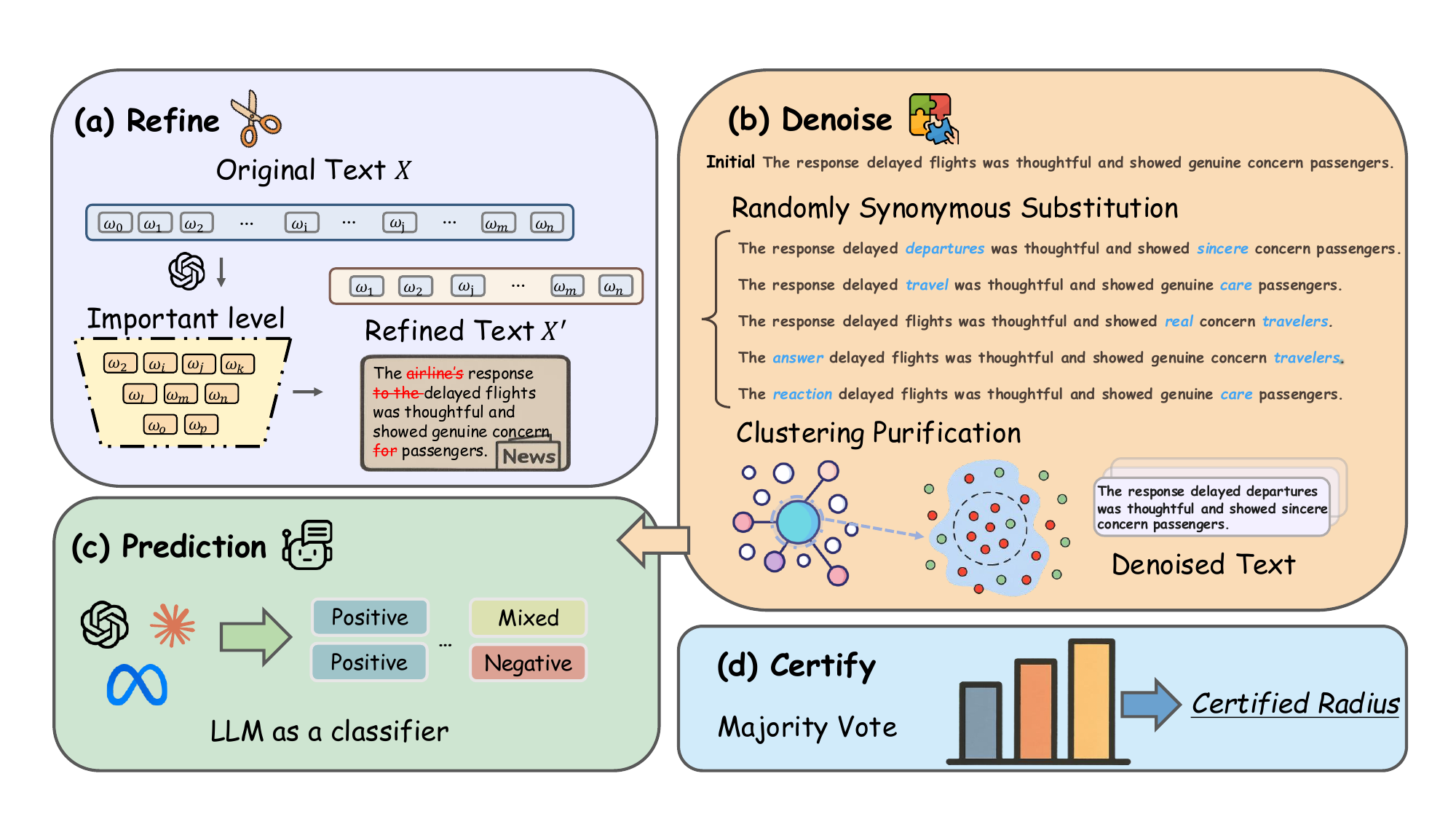}
    \caption{
Overview of our certified robustness framework \textbf{CluCERT}. (a) \textbf{Refine} removes irrelevant tokens using LLM generation to improve efficiency. (b) \textbf{Denoise} generates adversarial variants via synonym substitution and applies semantic clustering for purification. (c) \textbf{Predict} uses different LLMs for classification and aggregates outputs via majority vote. (d) \textbf{Certify} computes the certified radius based on the voting outcome. }
    \label{fig:example} % 图片标签（用于引用）
\end{figure*}
\section{Related Works}
\textit{Randomized smoothing}\cite{cohen2019certifiedadversarialrobustnessrandomized} has become a mainstream approach in certified robustness research due to its ability to provide probabilistic guarantees without requiring access to model internals. The core idea is to ensure that a model's prediction remains unchanged under perturbations within a predefined set (e.g., modifying up to $n$ pixels), thereby establishing provable robustness bounds. To enhance the certified robustness bounds of randomized smoothing,  \citet{salman2020denoised} first introduced the concept of \textit{``denoising''smoothing} in the context of computer vision. \citet{carlini2022certified} further proposed the integration of diffusion models to remove the added Gaussian noise during the smoothing process, thereby strengthening robustness guarantees for standard models. However, prior work has mainly focused on the image domain. In this work, we extend the denoising concept to text by introducing a semantic consistency mechanism that selects perturbations aligned with the original meaning. In the vision domain, \citet{levine2020robustness} proposed a random ablation method for defending against sparse adversarial attacks. The approach involves randomly removing pixels from an image before classification to assess their influence on the output. \citet{jia2022tightl0normcertifiedrobustness} extended this idea to the top-$k$ prediction setting. %where a subset of input features is randomly retained, while the rest are set to constant values such as the mean or median. 
Inspired by these methods in image-based tasks, ~\citet{zeng2021certifiedrobustnesstextadversarial} transferred the pixel deletion strategy to text classification by randomly replacing words with \texttt{[MASK]} tokens. However, their approach was developed for small-scale models and lacks scalability to LLMs. ~\citet{ji2024advancingrobustnesslargelanguage} introduced LLMs to fill in masked tokens for semantic denoising. Nonetheless, their method suffers from two major issues: first, the substitutions are not verified, which may introduce semantic drift; second, invoking LLMs for large-scale denoising leads to considerable computational cost.

Our work aims to strike a balance between the loose robustness bounds typically associated with standard models and the high computational overhead of the denoising process required for smoothed models. By incorporating a semantic clustering filter and an efficient perturbation generation strategy, we improve robustness certification while reducing the overall computational burden.

\section{Background}
\subsection{Notation}

We consider the standard setting of text classification, where the input is a sentence represented as a sequence of discrete tokens $w = (w_1, w_2, \ldots, w_n) \in \mathcal{W}$. Each token $w_i$ is drawn from a finite vocabulary $\mathcal{S}$, and $n$ denotes the sentence length. The goal is to assign an input sentence $w$ to one of the target classes in a finite label set $\mathcal{Y}$. A \textit{soft classifier} is defined as a mapping $F: \mathcal{W} \rightarrow \Delta^{|\mathcal{Y}|}$, which outputs a probability distribution over labels, where $\Delta^{|\mathcal{Y}|}$ is the $|\mathcal{Y}|$-dimensional probability simplex. The corresponding \textit{hard classifier} is given by $f(w) := \arg\max_{c \in \mathcal{Y}} F(w)_c$, returning the most likely label.

\subsection{Robustness Certification}

Certified robustness provides formal guarantees on a model's prediction stability under bounded perturbations, independent of specific attack strategies. Formally, let $f: \mathcal{W} \rightarrow \mathcal{Y}$ be a text classifier, and let $w \in \mathcal{W}$ be a clean input sentence. A certification algorithm computes a certified radius $d \in \mathbb{N}$ such that for any perturbed sentence $w' \in \mathcal{W}$ satisfying $\|w - w'\|_0 \leq d$, the prediction remains unchanged, i.e., $f(w') = f(w)$. Here, $\|w - w'\|_0 = \sum_{i=1}^n \mathbb{I}[w_i \neq w_i']$ denotes the Hamming distance, corresponding to the number of modified tokens.

This form of certification guarantees that the classifier remains robust against all adversarial inputs within an $\ell_0$ ball of radius $r$, regardless of how the perturbations are constructed. The certified radius serves as a conservative estimate of the model’s robustness; empirical studies~\cite{287372} have shown that robustness under the same perturbation budget is typically no worse and often higher.

\subsection{Randomized Smoothing on Text}

% In this work, we focus on achieving certified robustness through \textit{randomized smoothing}\cite{cohen2019certifiedadversarialrobustnessrandomized}, a technique that has shown strong theoretical guarantees in continuous domains. Inspired by prior work in the image domain~\cite{carlini2022certified,levine2020robustness,salman2020denoised}, we extend this idea to text classification by constructing a smoothed classifier based on discrete perturbations.

Let \( f: \mathcal{W} \rightarrow \mathcal{Y} \) be a base classifier that maps an input sentence \( w \in \mathcal{W} \) to a class label in \( \mathcal{Y} \). To improve its robustness, we construct a \textit{smoothed classifier} \( g \) by averaging the predictions of \( f \) over randomly perturbed input.

Each perturbation is generated by randomly masking some words in \( w \) and replacing them with semantically similar alternatives. Let \( T(w) \) denote a randomly perturbed version of \( w \). The smoothed classifier is defined as:$g(w) := \arg\max_{c \in \mathcal{Y}} \mathbb{P}[f(T(w)) = c]$,
where \( \mathbb{P}[f(T(w)) = c] \) represents the probability that the perturbed input is classified as label \( c \). For any class $c\in\mathcal{Y}$, we define its smoothed probability as:
 \( p_c(w) := \mathbb{P}_{\tilde{w} \sim \mathcal{D}(w)}[f(\tilde{w}) = c] \), where \( \mathcal{D}(w) \) reflects the randomness in the perturbation process.

By ensembling predictions over a distribution of semantically plausible inputs, randomized smoothing effectively transforms the base model into a robust classifier that can tolerate sparse word-level perturbations, making it suitable for formal certification.

\section{Methodology}

In this section, we present our theoretical contributions for certified robustness in text classification. Specifically, we begin by formally defining the two key operations in our framework, which correspond to the core processes of noise injection and semantic denoising.

Given a perturbation ratio parameter $m \in (0,1)$, we randomly select $s = \lfloor (1 - m) \cdot n \rfloor$ positions from a sentence of length $n$ as the retention set $\mathcal{T} \subseteq [n]$, and replace the remaining positions with the special symbol \texttt{[MASK]}. The mask operation is defined as $\mathcal{M}(w, \mathcal{T}) = (\tilde{w}_1, \ldots, \tilde{w}_n)$, where $\tilde{w}_i = w_i$ if $i \in \mathcal{T}$ and $\tilde{w}_i = \texttt{[MASK]}$ if $i \notin \mathcal{T}$. For example, let the input sentence be $w = (\texttt{A}, \texttt{B}, \texttt{C}, \texttt{D}, \texttt{E})$, and let $m = 0.4$, then $s = \lfloor 0.6 \cdot 5 \rfloor = 3$. Randomly selecting the retention set $\mathcal{T} = \{1, 3, 5\}$, the masked sentence becomes $\mathcal{M}(w, \mathcal{T}) = (\texttt{A},\ \texttt{[MASK]},\ \texttt{C},\ \texttt{[MASK]},\ \texttt{E})$. 

Second, we replace \texttt{[MASK]} positions with semantically similar words, such as synonyms, context-predicted words, or embedding-similar words. Let the semantic recovery mapping be $\mathcal{R} : \mathcal{W}_{\texttt{mask}} \to \mathcal{W}$, then the final perturbed sample is $T(w, \mathcal{T}) := \mathcal{R}(\mathcal{M}(w, \mathcal{T}))$. Continuing from the previous case, suppose the \texttt{[MASK]} positions are replaced with semantically similar \texttt{B'} and \texttt{D'} respectively, then $T(w, \mathcal{T}) = (\texttt{A},\ \texttt{B'},\ \texttt{C},\ \texttt{D'},\ \texttt{E})$, where \texttt{B'} and \texttt{D'} are substitutes for \texttt{B} and \texttt{D}, and the overall semantics remain consistent.

Thus, let the base classifier be $f:\mathcal{W}\to\mathcal{Y}$. We construct a smoothed classifier by aggregating predictions over randomly perturbed versions of the input. Formally, the smoothed classifier is defined as:
\begin{align}
g(w):=\arg\max_{y\in\mathcal{Y}}\mathbb{P}_{\mathcal{T}\sim\mathcal{U}_{n,s}}[f(T(w,\mathcal{T}))=y]
\end{align}
where $\mathcal{U}_{n,s}$ denotes the uniform distribution over all $\binom{n}{s}$ possible selections of $s$ positions from $n$ positions.

For any $c\in\mathcal{Y}$, the smoothed probability is defined as:
\begin{align}
p_c(w):=\mathbb{P}_{\mathcal{T}\sim\mathcal{U}_{n,s}}[f(T(w,\mathcal{T}))=c]
\end{align}

\subsection{Building a Smoothed LLM}

To build a smoothed classifier for textual tasks, and inspired by the insights from~\cite{jia2022tightl0normcertifiedrobustness} and~\cite{zeng2021certifiedrobustnesstextadversarial}, we propose a novel certification bound that differs from existing methods by incorporating both the sampling shift term~$\Delta_t$ and a semantic recovery stability factor~$\gamma$, enabling robustness certification under textual perturbations.

\newtheorem{theorem}{Theorem}
\begin{theorem}[\cite{levine2020robustness}]
For any $w'\in\mathcal{W}$ satisfying $\|w - w'\|_0 \leq d$, the smoothed probability for any class $c \in \mathcal{Y}$ satisfies
\begin{align}
|p_c(w) - p_c(w')| \leq \gamma \cdot \Delta_t,
\end{align}
where
\begin{align}
\Delta_t := 1 - \frac{\binom{n - d}{s}}{\binom{n}{s}},
\end{align}
and
\begin{align}
\gamma \in [0,1].
\end{align}
\end{theorem}

Here, \( d\) denotes the maximum number of perturbed words, i.e., the \( \ell_0 \)-distance between \( w \) and \( w' \).
 $\Delta_t = 1 - \binom{n - d}{s} / \binom{n}{s}$ quantifies the sampling distribution shift induced by $\ell_0$ perturbations, while $\gamma \in [0,1]$ captures the semantic recovery stability. A complete proof is provided in \textbf{Appendix A}.

In practice, we approximate the exact smoothed probability \( p_c(w) \) by averaging the model’s predictions over a finite number of randomly sampled mask patterns, since enumerating all possible masks is computationally infeasible. Specifically, we sample $N$ mask sets $\mathcal{T}_1, \ldots, \mathcal{T}_N \sim \mathcal{U}_{n,s}$ and estimate:
$\hat{p}_c(w) = \frac{1}{N} \sum_{i=1}^{N} \mathbb{I}[f(T(w, \mathcal{T}_i)) = c]$.

To provide rigorous guarantees despite finite sampling, we construct confidence intervals using the Clopper-Pearson method~\cite{clopper1934use}. Let $\underline{p}_c(w)$ and $\overline{p}_c(w)$ denote the lower and upper bounds of the $(1 - \alpha)$ confidence interval for $p_c(w)$. Our certification procedure (as shown in Algorithm~\ref{alg:semantic_smoothing}) leverages these conservative estimates to ensure probabilistic soundness.

\newtheorem{corollary}{Corollary}

\begin{corollary}

\label{cor:practical}
Let $c=\arg\max_y \underline{p}_y(w)$ be the class with highest lower confidence bound.
If for all $j\neq c$:
\begin{align}
\underline{p}_c(w) - \overline{p}_j(w) > 2\gamma\cdot\Delta_t
\end{align}
then for all $w'\in\mathcal{B}_t(w):=\{w':\|w-w'\|_0\leq t\}$, we have $g(w')=c$ with probability at least $1-\alpha$.
\end{corollary}

\begin{proof}
For any $w' \in \mathcal{B}_t(w)$ and $j \neq c$, the probability drift bound gives
\( p_c(w') \geq p_c(w) - \gamma\Delta_t \) and \( p_j(w') \leq p_j(w) + \gamma\Delta_t \). Thus,
\begin{align*}
p_c(w') - p_j(w') &\geq (p_c(w) - \gamma\Delta_t) - (p_j(w) + \gamma\Delta_t) \\
&= p_c(w) - p_j(w) - 2\gamma\Delta_t > 0
\end{align*}
which implies \( c = \arg\max_y p_y(w') \), i.e., \( g(w') = c \).
\end{proof}

This bound extends the classical randomized smoothing certification by explicitly incorporating both the sampling shift and the uncertainty in semantic recovery stability. In binary classification tasks, the condition simplifies to \( p_c(w) - \gamma \Delta_t > 0.5 \), allowing robustness certification for text classification under a confidence level of $(1 - \alpha )$.

Furthermore, to enhance the efficiency of our smoothing framework and ensure consistent processing across all modules, we introduce a \textit{Refine} operation that selects the top-\(L\) most informative tokens from a sentence based on LLM-derived importance scores.
\newtheorem{definition}{Definition}
\begin{definition}[Refine Operation]
Given an input sentence \( w = (w_1, w_2, \ldots, w_n) \), where \( w_i \) denotes the \(i\)-th token, we define the refine function \( R: \mathcal{W} \rightarrow \mathcal{W}_L \) as
\[
R(w) := \{ w_i \in w : \text{rank}_\sigma(w_i) \leq L \}.
\]
Here, \( \sigma: \mathcal{W} \rightarrow \mathbb{R}^n \) is an importance scoring function computed by a language model that assigns each token \( w_i \) a real-valued score \( \sigma(w_i) \). The operator \( \text{rank}_\sigma(w_i) \) denotes the position of \( w_i \) when all tokens in \( w \) are sorted in descending order of importance. The parameter \( L \) is the target output length, ensuring \( |R(w)| = L \) for all inputs with \( |w| \geq L \). We denote \( \mathcal{W}_L := \{ w \in \mathcal{W} : |w| = L \} \) as the space of refined sequences of fixed length.
\end{definition}
This operation extracts the core semantic content of the input while ensuring a uniform output length across all texts. All downstream modules in our framework operate on these fixed-length refined inputs. In practice, we achieve this by injecting a carefully designed prompt into the LLM to rank input words by importance and prune less informative ones accordingly (see \textbf{Appendix B} for details).

\begin{algorithm}[t]
\caption{Clustering-Guided Denoising Smoothing}
\label{alg:semantic_smoothing}
\begin{tabular}{ll}
1: & \textbf{procedure} Classifier$(w, f, s, N)$ \\
2: & \quad Sample $\mathcal{B} = \{w_1', \ldots, w_N'\}$ where $w_i' = T(w, \mathcal{T}_i)$\\
3: & \quad $\tilde{\mathcal{B}} \gets \{w' \in \mathcal{B} : \phi(w') \in \mathcal{C}_{\max}\}$ \quad $\triangleright$ \textit{Clustering} \\
4: & \quad counts$[c] \gets |\{w' \in \tilde{\mathcal{B}} : f(w') = c\}|$ for all $c \in \mathcal{Y}$ \\
5: & \quad \textbf{return} counts \\
\\
6: & \textbf{procedure} Predict$(w, f, s, N)$ \\
7: & \quad counts $\gets$ Classifier$(w, f, s, N)$ \\
8: & \quad $\hat{c} \gets \arg\max_{c} \mathrm{counts}[c]$ \quad $\triangleright$ \textit{Majority vote} \\
9: & \quad $\tilde{p}_{\hat{c}} \gets$ counts$[\hat{c}] / \sum_c$ counts$[c]$ \\
10: & \quad \textbf{return} $\hat{c}, \tilde{p}_{\hat{c}}$ \\
\\
11: & \textbf{procedure} Certify$(w, y, f, s, N, N', \gamma, \alpha)$ \\
12: & \quad $\hat{c}, \tilde{p}_{\hat{c}} \gets$ Predict$(w, f, s, N)$ \\
13: & \quad Let $n_A, n_B$ be counts of the top-2 classes \\
14: & \quad \textbf{if} $\mathrm{BinomPValue}(n_A, n_A + n_B, 0.5) > \alpha$ \\
15: & \quad\quad \textbf{return} ABSTAIN \\
16: & \quad counts $\gets$ Classifier$(w, f, s, N')$\\
17: & \quad Compute bounds $\{\underline{p}_c, \overline{p}_c\}_{c \in \mathcal{Y}}$ from counts \\
18: & \quad $r^* \gets \max\{t : \underline{p}_y - \overline{p}_j > 2\gamma\Delta_t, \forall j \neq y\}$ \\
19: & \quad \textbf{return} $r^*$ \\
\end{tabular}
\end{algorithm}

\subsection{Cluster-guided Denoising}

In this section, we formally define the concept of clustering and provide a theoretical proof for how clustering-guided denoising smoothing improves certified robustness.

Let $\mathcal{W}_t(w)=\{w_1',\ldots,w_N'\}$ denote the perturbation set obtained through mask and semantic filling. 
We introduce an embedding function $\phi:\mathcal{W}\to\mathbb{R}^m$ that maps each perturbed sample to its semantic representation $z_i=\phi(w_i')$. 
By applying a clustering algorithm to partition $\{z_i\}_{i=1}^N$ into $K$ clusters $\mathcal{C}_1,\ldots,\mathcal{C}_K$, 
we identify the largest cluster index $k_{\max}:=\arg\max_{k}|\mathcal{C}_k|$ and denote its corresponding cluster as $\mathcal{C}_{k_{\max}}$. 
This induces a refined perturbation subspace $\tilde{\mathcal{W}}_t(w):=\{w_i'\mid\phi(w_i')\in\mathcal{C}_{k_{\max}}\}$, 
over which we re-estimate the smoothed probabilities as 
$\tilde{p}_c:=\frac{|\{w'\in\tilde{\mathcal{W}}_t(w):f(w')=c\}|}{|\tilde{\mathcal{W}}_t(w)|}$. 
We select the largest cluster as the primary semantic group, as semantically consistent perturbations tend to concentrate within it, 
whereas adversarial or noisy samples are typically dispersed among smaller clusters.

\begin{theorem}
\label{thm:clustering}
Let 
\begin{align}
    r^* = \max\left\{ t \in \mathbb{N} : p_c(w) - \max_{j \neq c} p_j(w) > 2\gamma\Delta_t \right\}
\end{align}

be the original certified radius. Suppose semantic clustering induces probability shifts such that there exists some \( \epsilon > 0 \) satisfying:
\begin{align*}
\tilde{p}_c(w) &\geq p_c(w) + \epsilon, \\
\tilde{p}_j(w) &\leq p_j(w) - \epsilon \quad \text{for all } j \neq c.
\end{align*}
Then the certified radius with clustering $\tilde{r}^*$ satisfies $\tilde{r}^* > r^*$
\end{theorem}

The complete proof is provided in \textbf{Appendix C}. We also provide the detailed clustering algorithm implementation in \textbf{Appendix D}. This theorem shows that denoising can provably improve the certification bound. By leveraging clustering to focus on semantically coherent perturbations, we reduce the influence of low-consistency perturbations and shift the prediction distribution towards the correct class, thus enabling tighter robustness guarantees.

\newtheorem{lemma}{Lemma}
\begin{lemma}
\label{lem:variance}
If the embedding mapping $\phi$ and classifier $f$ satisfy the Lipschitz condition:
\begin{align*}
\|\phi(w_i')-\phi(w_j')\| \leq \rho \Rightarrow \mathbb{P}[f(w_i')=f(w_j')] \geq 1-L\rho
\end{align*}
and the clustering diameter satisfies $\text{diam}(\phi(\tilde{\mathcal{W}}_t(w))) \leq r$, then:
\begin{align}
\text{Var}_{w'\in\tilde{\mathcal{W}}_t(w)}[\mathbb{I}[f(w')=c]] \leq Lr
\end{align}
\end{lemma}
\begin{proof}
Let $\mu_c = \mathbb{P}_{w'\in\tilde{\mathcal{W}}_t(w)}[f(w') = c]$ denote the class probability within the refined cluster. Since the semantic diameter of the cluster is at most $r$, for any two samples $w_i', w_j' \in \tilde{\mathcal{W}}_t(w)$, we have $\|\phi(w_i') - \phi(w_j')\| \leq r$, and thus by the Lipschitz condition,
\[
\mathbb{P}[f(w_i') = f(w_j')] \geq 1 - Lr.
\]

Now fix an arbitrary reference point $w_0 \in \tilde{\mathcal{W}}_t(w)$. If $f(w_0) = c$, then for all $w' \in \tilde{\mathcal{W}}_t(w)$ we have $\mathbb{P}[f(w') = c] \geq 1 - Lr$, implying $\mu_c \geq 1 - Lr$. Conversely, if $f(w_0) \neq c$, then $\mathbb{P}[f(w') \neq c] \geq 1 - Lr$, so $\mu_c \leq Lr$.

In either case, the variance of the binary variable $\mathbb{I}[f(w') = c]$ satisfies
\[
\text{Var}[\mathbb{I}[f(w') = c]] = \mu_c(1 - \mu_c) \leq Lr.
\]
\end{proof}

Building on Lemma~\ref{lem:variance}, we observe that semantic clustering not only increases the mean prediction probability for the majority class but also reduces its variance. When perturbed samples are tightly grouped in the embedding space, the Lipschitz continuity of the classifier implies more stable predictions. This stability leads to lower variance and, consequently, enables tighter and more reliable certified robustness bounds.

In conclusion, we interpret cluster-guided denoising as projecting the perturbation set $\mathcal{W}_t(w)$ onto a semantically consistent subspace $\tilde{\mathcal{W}}_t(w) = \Pi_{\mathcal{S}_{\text{sem}}}(\mathcal{W}_t(w))$. The target subspace $\mathcal{S}_{\text{sem}}$ is characterized by: (i) tight clustering with $\max_{w',w''\in\mathcal{S}_{\text{sem}}}|\phi(w')-\phi(w'')|\leq r$, and (ii) high prediction consistency where $\mathbb{P}[f(w')=f(w'')] \geq 1-Lr$ for any $w',w'' \in \mathcal{S}_{\text{sem}}$. This projection acts as semantic denoising, improving model stability and enlarging certified robustness regions without modifying the base model.

\section{Experiments}

\subsection{Experimental Setup} 
\subsubsection{Overview} 
We consider two major task settings to evaluate our proposed framework CluCERT:
\textbf{\textit{(a)}} certified robustness of a given LLM under word substitutions,  
and \textbf{\textit{(b)}} empirical robustness under attacks.

\subsubsection{Datasets and models}

\begin{figure*}[ht]
    \centering
    \includegraphics[width=0.9\textwidth ]{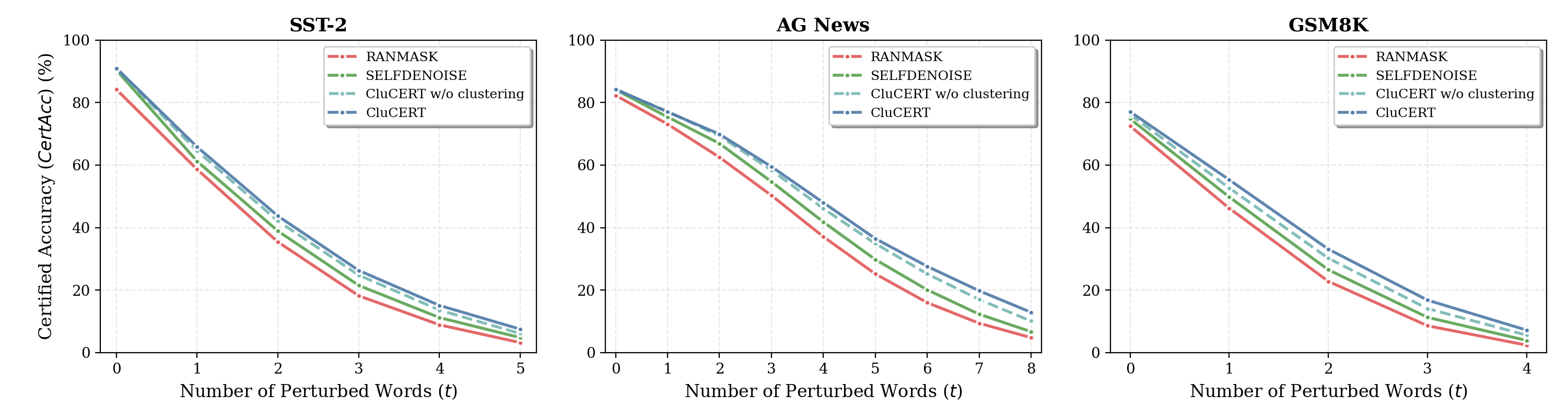}
    \caption{
    Certified accuracy on SST-2, AG News, and GSM8K under different numbers of perturbed words.
}
    \label{fig:exe}
\end{figure*}
To demonstrate both the certified radius enhancement our framework (settings \textbf{\textit{a}}), we follow prior work \cite{ji2024advancingrobustnesslargelanguage} and conduct experiments on two widely used sentiment classification datasets: SST-2 \cite{socher2013recursive} and AGNews \cite{zhang2015character}. We additionally include the GSM8K dataset \cite{cobbe2021gsm8k} to assess robustness in mathematical problem-solving. For these evaluations, we use ChatGPT-3.5 \footnote{\url{https://platform.openai.com}} as base models. To assess empirical robustness under adversarial attacks (setting \textbf{\textit{b}}), we adopt Textbugger \cite{Li_2019} and DeepWordBug \cite{gao2018blackboxgenerationadversarialtext}, using the open-source LLM Vicuna \footnote{\url{https://huggingface.co/lmsys/vicuna-7b-v1.5}} for evaluation.

\subsubsection{Implementation details and metrics} 
Following prior work \cite{zeng2021certifiedrobustnesstextadversarial}, we randomly selected 2,000 examples from each dataset for both certified and empirical robustness experiments. For each input instance, we generated 1,000 perturbed samples using word substitutions. All experiments are conducted under a consistent prompt interface to ensure a fair comparison across defense strategies.

To enable efficient local synonym substitution, we construct a hybrid candidate pool centered on WordNet \footnote{https://wordnet.princeton.edu/}, augmented with embedding-based nearest neighbors and domain-specific lexicons. After generating perturbed sentences, we compute sentence-level semantic similarity between each perturbed sentence and the original input using a pre-trained BERT model. We employ a semantic similarity threshold $\tau$ to filter perturbations. Here, $\tau$ is task-specific and determined through validation. The confidence level for certification is set to \( \alpha = 0.05 \).

For certified robustness evaluation, our primary metrics include the average certified radius \( \mathit{r}_{\text{avg}} \) and the certified accuracy under varying perturbation levels. Certified accuracy at level \( \delta \) is defined as \( \text{CertAcc}(\delta) = \frac{1}{N} \sum_{j=1}^{N} \mathbb{I}(r_j \geq \delta) \), where \( r_j \) is the certified radius of the \( j \)-th input. For empirical robustness evaluation, we use the attack success rate (ASR) against multiple adversarial attack strategies. Additionally, to assess the efficiency, we use the execution time \( \mathit{t} \) and the token cost \( \mathit{c} \) as metrics, which reflect the computational and monetary cost when querying LLMs.

\subsubsection{Baselines}
We compare with two baselines: RanMASK \cite{zeng2021certifiedrobustnesstextadversarial}, a randomized smoothing method without denoising, and SelfDenoise \cite{ji2024advancingrobustnesslargelanguage}, which applies denoising but lacks verification. These highlight the effects of our denoising and clustering modules, respectively. For consistency, both baselines are evaluated with the same \textit{Refine} module as our method (denoted with an asterisk \textbf{*}).

\subsection{Certified Robustness Evaluation}

While prior methods rely on white-box access and often require fine-tuning or retraining, CluCERT instead operates in a black-box setting, interacting with the model solely through API queries. 
\begin{table}[ht]
\centering

\small
\begin{tabular}{lcccccc}
\toprule
\textbf{Method} & \multicolumn{2}{c}{\textbf{SST-2}} & \multicolumn{2}{c}{\textbf{AGNews}} & \multicolumn{2}{c}{\textbf{GSM8K}} \\
\cmidrule(lr){2-3} \cmidrule(lr){4-5} \cmidrule(lr){6-7}
& \( \mathit{r}_{\text{avg}} \) & $Coe$ & \( \mathit{r}_{\text{avg}} \) & $Coe$ &\( \mathit{r}_{\text{avg}} \) & $Coe$ \\
\midrule
RanMASK             & 1.24 & 1.12 & 2.78 & 0.88 & 0.80 & 1.35 \\
SelfDenoise         & 1.38 & 1.07 & 3.08 & 0.82 & 0.92 & 1.25 \\
CluCERT(-Clu)            & 1.51 & 1.02 & 3.38 & 0.79 & 1.05 & 1.15 \\
\textbf{CluCERT}    & \textbf{1.58} & \textbf{1.00} & \textbf{3.51} & \textbf{0.79} & \textbf{1.16} & \textbf{1.07} \\
\bottomrule
\end{tabular}
\caption{
Robustness certificates across datasets. CluCERT(-Clu) refers to our framework without clustering. 
\( \mathit{r}_{\text{avg}} \) denotes the average certified radius (larger indicates stronger robustness), and 
\( \mathit{Coe} \) is the coefficient of variation (smaller indicates higher stability).
}
\label{tab::exe}
\end{table}
This design makes it applicable to closed-source large language models and more reflective of real-world deployment. Specifically, we use ChatGPT-3.5 as the base model and evaluate its certifiable robustness without accessing model weights. In contrast to vision tasks ~\cite{cohen2019certifiedadversarialrobustnessrandomized,carlini2022certified} that typically assume continuous perturbations, textual perturbations are inherently discrete, such as word substitutions. Moreover, certified radii in text are integer-valued, indicating the maximum number of word substitutions that preserve the model’s prediction. To capture robustness under this discrete setting, we report the average certified radius (\( \mathit{r}_{\text{avg}} \)), along with the coefficient of variation ($Coe$), to characterize both the level and stability of model robustness across different perturbation budgets.

Beyond standard text classification tasks, we extend CluCERT to mathematical question answering, a setting that poses greater challenges for robustness due to its deterministic outputs and high sensitivity to input perturbations. As shown in Table~\ref{tab::exe}, CluCERT achieves the highest \( \mathit{r}_{\text{avg}} \) and the lowest $Coe$ across all evaluated datasets, indicating superior and stable certified robustness. On GSM8K, CluCERT attains an \( r_{\text{avg}} \) of 1.16, outperforming RanMASK (0.80) and SelfDenoise (0.92), with the smallest $Coe$ of 1.07. This demonstrates its ability to maintain robustness under structurally complex and fragile input conditions. On AGNews, which features longer inputs and greater semantic diversity, CluCERT again outperforms baselines, achieving \( r_{\text{avg}} = 3.51 \) and $Coe$ = 0.79, showing tolerance to stronger perturbations and consistency across examples. SST-2 exhibits lower certified radii, likely due to shorter inputs and the reliance on a few key emotional tokens, yet CluCERT still delivers the best performance, highlighting its adaptability to different task types.

To further understand the effect of the clustering mechanism, we conduct ablation studies focusing on the clustering-based denoising module. As shown in Figure~\ref{fig:exe}, integrating the synonym clustering strategy consistently improves certified accuracy across various perturbation levels, with particularly notable gains on long-text datasets such as AGNews. These results suggest that clustering mitigates semantic drift introduced by low-quality substitutions and enhances prediction stability. Compared to the variant without clustering, CluCERT yields both higher certified accuracy and smoother robustness curves throughout the entire perturbation range. Importantly, these empirical observations align with our theoretical analysis, confirming the effectiveness of clustering in balancing semantic preservation with perturbation diversity.

\subsection{Empirical Robustness under Attacks}

\begin{table}[t]
\centering
\footnotesize
\setlength{\tabcolsep}{3pt}
\begin{tabular}{@{}lccc@{}}
\toprule
\textbf{AGNews} & Clean Acc. & ASR (TB) & ASR (DWB) \\
\midrule
RanMask              & 82.3 & 47.1 & 42.6 \\
SelfDenoise          & 84.1 & 34.4 & 30.9 \\
CluCERT (w/o Refine) & \textbf{85.0} & 32.7 & 28.5 \\
CluCERT              & 84.3 & \textbf{29.8} & \textbf{26.9} \\
\midrule
\textbf{SST-2} & Clean Acc. & ASR (TB) & ASR (DWB) \\
\midrule
RanMask              & 84.2 & 52.3 & 47.4 \\
SelfDenoise          & 90.4 & 44.6 & 35.7 \\
CluCERT (w/o Refine) & \textbf{91.0} & 45.5 & 33.3 \\
CluCERT              & \textbf{91.0} & \textbf{41.3} & \textbf{31.2} \\
\bottomrule
\end{tabular}
\caption{Clean accuracy and attack success rate (\%) under TextBugger (TB) and DeepWordBug (DWB) attacks on AGNews and SST-2. Lower ASR indicates stronger robustness.}
\label{tab:empirical}
\end{table}

We evaluate CluCERT under two representative black-box attacks, TextBugger and DeepWordBug, using the default settings of the TextAttack toolkit \cite{morris2020textattack}. Experiments are conducted on two widely used classification benchmarks, SST-2 and AGNews. RanMASK and SelfDenoise serve as baselines, and we additionally include a variant of CluCERT without the proposed \textit{Refine} module. Table~\ref{tab:empirical} reports clean accuracy and attack success rate (ASR), where lower ASR indicates stronger robustness.

CluCERT consistently achieves the lowest attack success rates (ASR) across both datasets and attack types, demonstrating robust and stable empirical performance. On SST-2, it maintains the highest clean accuracy (91.0\%) while achieving the lowest ASR, clearly outperforming all baseline methods. On AGNews, CluCERT also exhibits the strongest robustness under adversarial attacks. Although its clean accuracy slightly decreases compared to the variant without the \textit{Refine} module, this trade-off results in a substantial gain in robustness. We attribute this improvement primarily to the \textit{Refine} module. By compressing the input, filtering peripheral content, and emphasizing core semantics, this module helps the model focus on key information, thereby enhancing robustness and prediction stability. While the process may occasionally remove marginally informative tokens or introduce output variance induced by prompt design, it generally reduces input length and minimizes noisy information. This leads to improved resistance to perturbations and increased inference efficiency.

The integration of clustering-based substitution and the \textit{Refine} module enables CluCERT to strike a balance between clean accuracy and adversarial robustness. The consistent drop in ASR across datasets and attack types confirms the effectiveness and generalizability of our approach.

\subsection{Efficiency}

In this section, we analyze the efficiency of the proposed method. As shown in Figure~\ref{fig:time}, we present the estimated time cost of each processing stage when generating 1000 perturbed samples for a single input. Our method achieves approximately 6.8$\times$ speedup compared to the baseline method, significantly reducing the computational overhead for deployment. In addition, the LLM-based substitution stage incurs non-negligible token-level inference costs, which may lead to substantial economic burdens in large-scale text processing scenarios.

\begin{figure}[ht]
    \centering
    \includegraphics[width=0.45\textwidth,trim=0 0 0 60, clip]{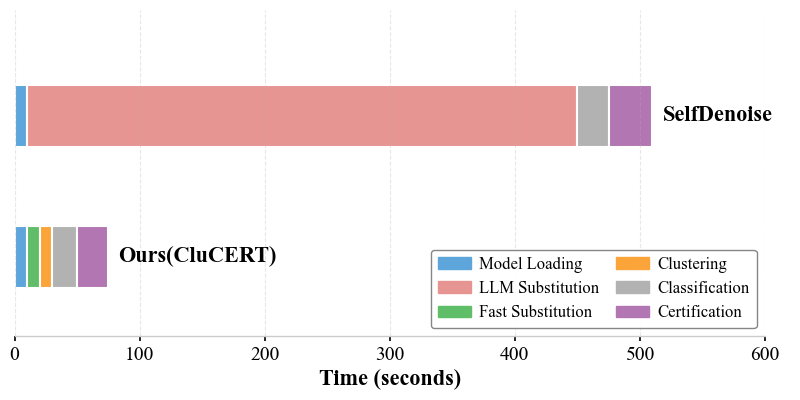}
    \caption{Time cost for each mode from SelfDenoise and CluCERT
}
    \label{fig:time}
\end{figure}

The speedup primarily stems from our fast synonym substitution strategy, which avoids costly LLM-based generation used in prior methods such as SelfDenoise. Overall, our approach achieves a favorable balance between robustness and computational efficiency.

\section{Conclusion}
We propose CluCERT, a clustering-guided and efficient smoothing framework for certifying the robustness of large language models. By combining fast synonym-based perturbation with semantic clustering, CluCERT focuses on meaningful substitutions while improving sampling efficiency, enabling stronger certified bounds with significantly lower computational cost. Extensive experiments on text classification and mathematical reasoning tasks demonstrate its effectiveness and generalizability.
\section*{Acknowledgements}
This work was supported by The Royal Society Grant (Ensuring Trustworthy AI: Robustness Certification for Large Language Models)[Reference RGS\textbackslash R2\textbackslash 252444].
GJ's contribution is supported by the NVIDIA Academic Grant Program. 
The authors would like to acknowledge the use of the University of Exeter High-Performance Computing (HPC) facility in carrying out this work.

\bibliography{aaai2026}

\clearpage
\onecolumn
\appendix
\newpage
\section{Robustness Bound (Theorem 1)}
Our robustness bound extends randomized smoothing to text classification under $\ell_0$-bounded adversarial word substitutions. To model semantic variability in language, we introduce a recovery factor $\gamma$ that quantifies prediction stability under meaning-preserving perturbations. Below, we derive Theorem 1 under this adapted setting.

\begin{proof}
Let $I=\{i:w_i\neq w'_i\}$ be the set of differing positions. By condition, $|I|\leq d$. Throughout, we assume the recovery operator $T(\cdot,\cdot)$ is deterministic given the mask set $\mathcal{T}$.
Define event $A=\{\mathcal{T}\cap I=\emptyset\}$, i.e., the retention set does not include any differing positions.
When event $A$ occurs, for all $i\in\mathcal{T}$, we have $w_i=w'_i$, thus $T(w,\mathcal{T})=T(w',\mathcal{T})$, and consequently $f(T(w,\mathcal{T}))=f(T(w',\mathcal{T}))$.
Calculate $\mathbb{P}[A]$:
\begin{align*}
&\mathbb{P}[\mathcal{T}\cap I=\emptyset]\\
&=\frac{|\{\mathcal{T}:|\mathcal{T}|=s,\,\mathcal{T}\subseteq\{1,\ldots,n\}\setminus I\}|}{|\{\mathcal{T}:|\mathcal{T}|=s\}|}\\
&=\frac{\binom{n-|I|}{s}}{\binom{n}{s}}\geq\frac{\binom{n-d}{s}}{\binom{n}{s}},
\end{align*}
and hence
\begin{align*}
\mathbb{P}[A^c]\leq 1-\frac{\binom{n-d}{s}}{\binom{n}{s}}=\Delta_t.
\end{align*}
Consider the probability difference. Since the distribution of $\mathcal{T}$ does not depend on the input, both smoothed probabilities can be written as expectations over the same $\mathcal{T}$:
\begin{align*}
&|p_c(w)-p_c(w')|\\
&=|\mathbb{E}_{\mathcal{T}}[\mathbb{I}[f(T(w,\mathcal{T}))=c]\\
&\quad-\mathbb{I}[f(T(w',\mathcal{T}))=c]]|\\
&\leq\mathbb{E}_{\mathcal{T}}[|\mathbb{I}[f(T(w,\mathcal{T}))=c]\\
&\quad-\mathbb{I}[f(T(w',\mathcal{T}))=c]|]\\
&\leq\mathbb{P}_{\mathcal{T}}[f(T(w,\mathcal{T}))\neq f(T(w',\mathcal{T}))],
\end{align*}
where the last step holds since the absolute difference of the two indicators equals $1$ only if the two predictions differ.
Define the semantic recovery stability factor:
\begin{align*}
\gamma:=\sup_{\substack{w,w',\mathcal{T}:\,\|w-w'\|_0\leq d,\\ \mathcal{T}\cap I\neq\emptyset}}\mathbb{P}\big[f(T(w,\mathcal{T}))\neq f(T(w',\mathcal{T}))\big]\in[0,1].
\end{align*}
Decompose this probability:
\begin{align*}
&\mathbb{P}[f(T(w,\mathcal{T}))\neq f(T(w',\mathcal{T}))]\\
&=\mathbb{P}[f(T(w,\mathcal{T}))\neq f(T(w',\mathcal{T}))\,|\,A]\cdot\mathbb{P}[A]\\
&\quad+\mathbb{P}[f(T(w,\mathcal{T}))\neq f(T(w',\mathcal{T}))\,|\,A^c]\cdot\mathbb{P}[A^c]\\
&\leq 0\cdot\mathbb{P}[A]+\gamma\cdot\mathbb{P}[A^c]\\
&\leq\gamma\cdot\left(1-\frac{\binom{n-d}{s}}{\binom{n}{s}}\right)=\gamma\cdot\Delta_t,
\end{align*}
where the first conditional probability vanishes since $f(T(w,\mathcal{T}))=f(T(w',\mathcal{T}))$ on event $A$, and the second is bounded by the definition of $\gamma$.
Therefore, we obtain:
\begin{align*}
|p_c(w)-p_c(w')|\leq\gamma\cdot\Delta_t,
\end{align*}
\end{proof}

% Following prior work on estimating semantic stability factors such as $\beta$, we approximate the recovery factor $\gamma$ in Theorem~1 using the smoothed prediction probability $p_c(x)$. When the model consistently predicts class $c$ across masked versions of $x$, it indicates robustness to meaning-preserving perturbations. Empirically, we observe that when the number of samples used to estimate $p_c(x)$ is sufficiently large, the difference between $\gamma$ and $p_c(x)$ becomes negligible. Hence, we adopt:
% $\gamma \approx p_c(x)$
% This approximation simplifies computation and remains stable in practice.

Following prior work on estimating semantic stability factors such as $\beta$, we approximate the recovery factor $\gamma$ in Theorem~1 using the smoothed prediction probability $p_c(x)$. For certification only the one-sided deviation matters, which is bounded by $\beta:=\mathbb{P}[f(T(x,\mathcal{T}))=c\,|\,\mathcal{T}\cap I\neq\emptyset]\cdot\Delta_t$, depending only on the original input $x$. Assuming the overlap event is approximately independent of the model's prediction on the clean input, we have $\beta\approx p_c(x)$. Hence, we adopt:
$\gamma \approx p_c(x)$
This approximation simplifies computation and remains stable in practice, yielding a practical certificate under the independence assumption.
\section{Prompts and Instructions}

\subsection{Refine Operation}

To identify the most semantically important words for input reduction, we adopt a composite prompting strategy. Specifically, we design a set of diverse prompts that elicit word-level importance judgments from a language model. Each prompt is phrased differently to encourage varied perspectives on which words are most critical to the sentence’s sentiment or meaning. By querying the model multiple times with these prompts, we obtain several ranked lists of word importance. The results are then aggregated, and high-ranking words are selected to construct the refined input. Specifically, Table~\ref{tab:importance-prompts} illustrates the set of representative prompts we used to query the language model for word-level importance judgments. Each prompt is carefully designed to approach the concept of “importance” from slightly different angles, such as sentiment contribution, semantic relevance, or overall meaning. This diversity helps capture a broader perspective on which words are critical or negligible for the sentence’s meaning and ultimately improves the robustness of the refinement step.
\begin{table}[ht]
\centering
\begin{tabular}{p{0.95\linewidth}}
\toprule
\textbf{Representative Prompts for 4-Level Word Importance Classification} \\
\midrule
1. Please classify each word in the sentence into one of four levels based on its importance to the overall sentiment: Very Important, Important, Less Important, or Not Important. \\
\midrule
2. Assign an importance level (Very Important, Important, Less Important, Not Important) to each word depending on how strongly it affects the sentiment of the sentence. \\
\midrule
3. For every word in the sentence, evaluate its influence on the emotional tone and categorize it as either Very Important, Important, Less Important, or Not Important. \\
\midrule
4. Which words carry the core sentiment of the sentence? Please group all words into four categories according to their emotional contribution. \\
\midrule
5. Label each word in the sentence as Very Important, Important, Less Important, or Not Important, based on how necessary it is for understanding the sentence's meaning. \\
\midrule
6. Which words in this sentence are the least meaningful or most negligible in terms of semantic contribution? Please assign all words to four importance levels accordingly. \\
\bottomrule
\end{tabular}
\caption{Examples of prompts used to elicit 4-level word importance judgments from a language model.}
\label{tab:importance-prompts}
\end{table}

\subsection{Downstream prompts}
Below are the actual prompts and instructions we used in downstream tasks. 
\begin{tcolorbox}[
  colback=gray!5!white,     % 浅灰色背景
  colframe=gray!70!black,   % 深灰色边框
  title=Prompt Template Used for ChatGPT-3.5,
  fonttitle=\bfseries,      % 加粗标题
  boxrule=0.5pt,            % 边框粗细
  arc=2pt,                  % 圆角弧度
  left=5pt, right=5pt, top=5pt, bottom=5pt % 内边距
]
Below is an instruction that describes a task, followed by an input text. Respond appropriately by completing the task according to the instruction.

\vspace{1em}

\texttt{\#\#\#Instruction:} \\

\vspace{1em}

\texttt{\#\#\#Input:} \\
\{\textit{Insert sentence or input text here}\}

\vspace{1em}

\texttt{\#\#\#Response:}
\vspace{1em}
\end{tcolorbox}
The following descriptions are used to fill in the Instruction section of the prompt. The Input section should be filled with different input texts accordingly.

\newcounter{instruction}

\refstepcounter{instruction}\label{inst:sst2}
\noindent\textbf{Instruction \theinstruction:} Instruction used for \textbf{SST-2} Classification.\\
\textit{Below is an instruction that describes a sentiment classification task, followed by an English sentence as input. Respond with either \texttt{"positive"} or \texttt{"negative"} according to the sentence sentiment.}

\vspace{1em}

\refstepcounter{instruction}\label{inst:agnews}
\noindent\textbf{Instruction \theinstruction:} Instruction used for \textbf{AGNEWS} Topic Classification.\\
\textit{Below is an instruction for classifying a news article into one of four categories. Respond with the correct category name: \texttt{Sports}, \texttt{World}, \texttt{Technology}, or \texttt{Business}.}

\vspace{1em}

\refstepcounter{instruction}\label{inst:gsm8k}
\noindent\textbf{Instruction \theinstruction:} Instruction used for \textbf{GSM8K} Math Problem Solving.\\
Below is an instruction for solving a math word problem. Read the problem and return only the final numeric answer.

\section{Proof of Semantic Clustering for Certified Radius Improvement (Theorem 2)}

\begin{proof}
At the original certified radius $r^*$, the certification condition is satisfied with some margin:
\begin{align*}
p_c(w) - \max_{j \neq c} p_j(w) = 2\gamma\Delta_{r^*} + \delta
\end{align*}
where $\delta \geq 0$. Since $r^*$ is the maximum integer satisfying the certification condition, we have:
\begin{align*}
p_c(w) - \max_{j \neq c} p_j(w) \leq 2\gamma\Delta_{r^*+1}
\end{align*}

This implies:
\begin{align*}
0 \leq \delta \leq 2\gamma(\Delta_{r^*+1} - \Delta_{r^*})
\end{align*}

After semantic clustering, the probability gap becomes:
\begin{align*}
\tilde{p}_c(w) - \max_{j \neq c} \tilde{p}_j(w) &\geq (p_c(w) + \epsilon) - (\max_{j \neq c} p_j(w) - \epsilon)\\
&= p_c(w) - \max_{j \neq c} p_j(w) + 2\epsilon\\
&= 2\gamma\Delta_{r^*} + \delta + 2\epsilon
\end{align*}

To find the new certified radius $\tilde{r}^*$, we need the maximum $t$ such that:
\begin{align*}
2\gamma\Delta_{r^*} + \delta + 2\epsilon > 2\gamma\Delta_t
\end{align*}

Rearranging:
\begin{align*}
\Delta_t < \Delta_{r^*} + \frac{\delta + 2\epsilon}{2\gamma}
\end{align*}

For the asymptotic analysis, when $t \ll \min(s, n-s)$, we have:
\begin{align*}
\Delta_t = 1 - \prod_{i=0}^{s-1} \frac{n-t-i}{n-i} \approx \frac{st}{n-s/2}
\end{align*}

Thus, near $r^*$:
\begin{align*}
\Delta_{t} - \Delta_{r^*} \approx \frac{s(t-r^*)}{n-s/2}
\end{align*}

For $\Delta_t < \Delta_{r^*} + \frac{\delta + 2\epsilon}{2\gamma}$, we need:
\begin{align*}
\frac{s(t-r^*)}{n-s/2} < \frac{\delta + 2\epsilon}{2\gamma}
\end{align*}

Solving for $t$:
\begin{align*}
t < r^* + \frac{(n-s/2)(\delta + 2\epsilon)}{2\gamma s}
\end{align*}

Therefore, the new certified radius is:
\begin{align*}
\tilde{r}^* = r^* + \left\lfloor \frac{(n-s/2)(\delta + 2\epsilon)}{2\gamma s} \right\rfloor
\end{align*}

Since $\delta + 2\epsilon > 0$, we have $\tilde{r}^* > r^*$, completing the proof.
\end{proof}

\end{document}